\relax
\documentclass[a4paper,11pt]{article}

\usepackage{fullpage}
\usepackage[sort]{natbib}
\usepackage[colorlinks=true,linkcolor=blue, citecolor=blue]{hyperref}
\usepackage[nottoc,notlot,notlof]{tocbibind}
\usepackage{times}  
\usepackage{helvet} 
\usepackage{courier}  
\usepackage{url} 
\usepackage{graphicx}
\usepackage{subcaption}
\usepackage{lmodern}
\usepackage{datatool}
\usepackage{algorithm}
\usepackage[noend]{algpseudocode}

\usepackage{amsthm,amsmath,amssymb}
\usepackage{bm}
\usepackage{latexsym}
\usepackage{float}
\usepackage{color}
\usepackage{pifont}
\theoremstyle{plain}
\newtheorem{theorem}{Theorem}

\newtheorem{proposition}[theorem]{Proposition}

\newtheorem{corollary}[theorem]{Corollary}
\newcommand{\R}{\mathbb{R}}

\newcommand{\Epos}{\mathbb{E}_{\mathrm{p}}}
\newcommand{\Eneg}{\mathbb{E}_{\mathrm{n}}}
\newcommand{\Eunl}{\mathbb{E}_{\mathrm{u}}}
\newcommand{\pitest}{\pi'}
\newcommand{\ppos}{p_{\mathrm{p}}}
\newcommand{\pneg}{p_{\mathrm{n}}}
\newcommand{\zerooneloss}{\ell_{0\text{-}1}}

\newcommand{\feature}{\boldsymbol{x}}
\newcommand\numberthis{\addtocounter{equation}{1}\tag{\theequation}}
\newcommand{\ptest}{p_\mathrm{t}}
\newcommand{\pu}{p_\mathrm{u}}
\newcommand{\alphau}{\alpha_{\mathrm{unif}}}
\newcommand{\piu}{\pi_{\mathrm{unif}}}
\newcommand{\pig}{\pi^\mathrm{g}}
\newcommand{\dbquote}[1]{``#1''}

\usepackage{booktabs}       
\usepackage{nicefrac}       
\usepackage{microtype}      
\usepackage{multirow}
\usepackage{siunitx}

\newcolumntype{C}{>{$}c<{$}}
\newcolumntype{L}{>{$}l<{$}}

\begin{document}
\title{Positive-Unlabeled Classification under Class Prior Shift and Asymmetric Error}
\author{Nontawat Charoenphakdee$^1$ \and Masashi Sugiyama$^{2,1}$}
\date{
$^1$ The University of Tokyo 
$^2$ RIKEN
}
\maketitle
\begin{abstract}
Bottlenecks of binary classification from positive and unlabeled data (PU classification) are the requirements that given unlabeled patterns are drawn from the test marginal distribution, and the penalty of the false positive error is identical to the false negative error. 
However, such requirements are often not fulfilled in practice. 
In this paper, we generalize PU classification to the class prior shift and asymmetric error scenarios. 
Based on the analysis of the Bayes optimal classifier, we show that given a test class prior, PU classification under class prior shift is equivalent to PU classification with asymmetric error. 
Then, we propose two different frameworks to handle these problems, namely, a risk minimization framework and density ratio estimation framework. 
Finally, we demonstrate the effectiveness of the proposed frameworks and compare both frameworks through experiments using benchmark datasets.
\end{abstract}

\section{Introduction}
Learning from positive and unlabeled data is a learning problem where we are given positive (P) patterns and unlabeled (U) patterns, but negative patterns are not given. 
This learning problem is important when obtaining negative patterns is difficult due to high labeling costs or privacy concerns.
There are many problems in this category such as binary classification (PU classification)~\citep{denis1998pac,letouzey2000learning,elkan2008learning,ward2009presence,du2014,du2015convex, kiryo2017}, inlier-based outlier detection~\citep{hido2011statistical,blanchard2010semi}, matrix completion~\citep{hsieh2015pu}, data stream classification~\citep{li2009positive}, and time-series classification~\citep{nguyen2011positive}.
In this paper, we focus on PU classification. The goal is to learn a binary classifier that achieves low generalization error on the test distribution. Some motivating examples of PU classification are problems in the field of computational biology where we have test cases to verify whether a pattern is positive, but evidence after the test does not imply that a pattern is negative. Examples are gene-disease identification~\citep{yang2012positive} and conformational B-cell epitopes classification~\citep{ren2015positive}. Not only computational biology problems but PU classification can also be useful for other applications such as land-cover classification~\citep{li2011positive} and text classification~\citep{li2003learning}.

The main challenges of PU classification are lack of information about negative patterns and how to incorporate unlabeled patterns. A naive approach is to identify negative patterns from the given unlabeled patterns, and then apply a standard binary classification method with pseudo-negative patterns and positive patterns~\citep{li2003learning}, which relies on heuristics to identify negative patterns from unlabeled patterns. 
A more recent approach is to rewrite the misclassification risk so that we can minimize the risk using only positive and unlabeled patterns~\citep{du2014,du2015convex,kiryo2017}. This risk minimization approach has been demonstrated its effectiveness in PU classification and can be considered as the current state-of-the-art method.

However, a limitation of the risk minimization framework is the assumption that the test distribution is identical to the distribution that unlabeled patterns are drawn from. Unfortunately, this assumption might not always hold in practice. 
Intuitively, this assumption means that the ratio of positive data (class prior) in the unlabeled distribution is the same as the test distribution. Practically, it is more convenient if we can also collect unlabeled patterns from the source that has a different class prior. Another limitation of the current framework is that the existing formulation is limited to the setting where a false positive error and false negative error have the same penalty (symmetric error). There are real-world applications where the false positive penalty and false negative penalty are different. For example, a cancer detection task in medical diagnosis, false negative errors are life-threatening and therefore should have a higher penalty than false positive errors. As a result, both class prior shift and asymmetric error scenarios are highly relevant and it is important to investigate such scenarios so that we can apply machine learning techniques to support more real-world applications. To the best of our knowledge, both scenarios in PU classification have not been studied extensively yet.

The goal of this paper is to enable practitioners to apply PU classification to a wider range of applications. We consider two different frameworks namely, a risk minimization framework and density ratio framework, to handle PU classification under class prior shift and asymmetric error. Our main contributions are as follows:
\begin{itemize}
\item We prove that PU classification under class prior shift and PU classification with asymmetric error are equivalent. Moreover, we can handle PU classification even if both class prior shift and asymmetric error are considered simultaneously.
\item We propose a risk minimization framework for PU classification with class prior shift and asymmetric error based on the existing framework for binary classification from unlabeled data~\citep{lu2018minimal}.
\item We propose a density ratio framework for PU classification under class prior shift and asymmetric error.
\item We compare both frameworks in terms of convexity of the formulation, the need of retraining when the test condition changes (class prior shift or asymmetric error), and experiments using benchmark datasets. 
\end{itemize}

\section{Preliminaries}
In this section, we review the ordinary PU classification problem and density ratio estimation problem.
\subsection{Ordinary PU classification}
Let $\feature \in \R^d$ be a $d$-dimensional pattern and $y \in \{-1, 1\}$ be a label. We denote the class prior $p(y={+1})$ by $\pi$. In PU classification, we are given the positive patterns $X_\mathrm{p}$ drawn from the class conditional distribution over positive data and the unlabeled patterns $X_\mathrm{u}$ drawn from the marginal distribution over unlabeled data as follows:
\begin{align*}
 X_\mathrm{p}&:= \{\feature^\mathrm{p}_i\}_{i=1}^{n_\mathrm{p}} \stackrel{\mathrm{i.i.d.}}{\sim} p(\feature|y=+1) \text{,}\\ 
 X_\mathrm{u}&:=  \{\feature^\mathrm{u}_j\}_{j=1}^{n_\mathrm{u}} \stackrel{\mathrm{i.i.d.}}{\sim} \pi p(\feature|y=+1)+(1-\pi) p(\feature|y=-1) \text{,}
\end{align*}
where $p(\feature|y)$ is the conditional density of $\feature$ given y, $n_\mathrm{p}$ is the number of positive patterns and $n_\mathrm{u}$ is the number of unlabeled patterns. The existing framework assumes that the test patterns are drawn from the same distribution as unlabeled patterns~\citep{du2014, du2015convex, kiryo2017}. Similarly to the existing framework, we assume the class prior $\pi$ is known. In practice, it can be estimated using existing class prior estimation methods~\citep{prior3, prior2}. The goal is to find a classifier $g\colon \R^d \to \R$ that minimizes the expected misclassification risk with respect to the test distribution:
\begin{align}\label{pnrisk-decomposed}
R^{\zerooneloss}(g) &= \pi \Epos \left[ \zerooneloss(g(\feature)) \right] + (1-\pi) \Eneg \left[ \zerooneloss(-g(\feature)) \right]\text{,}
\end{align}
where $\zerooneloss(z) = -\frac{1}{2} \mathrm{sign}(z) + \frac{1}{2}$ is the zero-one loss, $\Epos[\cdot] $ and $\Eneg[\cdot]$ denote the expectations over $p(\feature|y=+1)$ and $p(\feature|y=-1)$, respectively. In practice, we use a \emph{surrogate loss} $\ell$, e.g., the squared loss $\ell(z)=(1-z)^2$, or the logistic loss $\ell(z)=\log(1+\exp(-z))$, to minimize a surrogate risk $R^{\ell}(g)$ instead of $\zerooneloss$ since minimizing Eq. \eqref{pnrisk-decomposed} is computationally infeasible~\citep{bartlett2006, zeroonenphard2,zeroonenphard1}. 

Note that there are other settings that can also be applied for PU classification~\citep{elkan2008learning,natarajan2013learning}, where a single set of unlabeled patterns is collected first and then some of the positive patterns in the set are labeled. Here, we prefer this two-sample setting of PU classification~\citep{du2014} because we can control the number of positive patterns and unlabeled patterns independently which is more realistic for practical applications~\citep{du2014,du2015convex,niu2016}.

Since negative patterns are not given, minimizing $R^{\ell}(g)$ using positive and unlabeled patterns is not directly possible. It has been shown that it is possible to rewrite the risk term $R^{\ell}(g)$ as follows~\citep{du2015convex}:
\begin{align}\label{purisk-decomposed}
R_{\mathrm{PU}}^{\ell}(g)  = \pi \Epos \left[ \ell(g(\feature)) - \ell(-g(\feature)) \right] + \Eunl \left[ \ell(-g(\feature)) \right] \text{,}
\end{align}
where $\Eunl[\cdot] $ denotes the expectation over the marginal density $p(\feature)$. An unbiased risk estimator based on Eq.~\eqref{purisk-decomposed} does not require negative patterns can be obtained immediately by replacing the expectations with sample averages as follows:
\begin{align*}
	\widehat{R}^{\ell}_{\mathrm{PU}}(g) &= \frac{\pi}{n_\mathrm{p}} \sum_{i=1}^{n_\mathrm{p}} \left[ \ell(g(\feature^\mathrm{p}_i)) - \ell(-g(\feature^\mathrm{p}_i)) \right] + \frac{1}{n_\mathrm{u}}  \sum_{i=1}^{n_\mathrm{u}}  \ell(-g(\feature^\mathrm{u}_i)). 
\end{align*}

Moreover, it has been shown that using a linear-in-parameter model, a convex formulation can be obtained if a loss $\ell$ satisfies the linear-odd condition~\citep{du2015convex, patrini2016loss}, i.e., $\ell(z)-\ell(-z)$ is linear. 
Examples of such loss functions are the squared loss, modified huber loss, and logistic loss.  However, this empirical risk of Eq. \eqref{purisk-decomposed} can be negative and can be prone to overfitting when using a highly flexible model. One way to cope with this problem is to apply a non-negative correction which enables the use of deep neural networks in PU classification~\citep{kiryo2017} as follows:
\begin{align*}
	\widehat{R}^{\ell}_{\mathrm{nnPU}}(g) &=  \frac{\pi}{n_\mathrm{p}} \sum_{i=1}^{n_\mathrm{p}}\ell(g(\boldsymbol{x}^\mathrm{p}_i)) + \max (0, \widehat{R}_{\mathrm{PU}\text{-}{\mathrm{neg}}}(g) ),
\end{align*}
where 
\begin{align*}
\widehat{R}_{\mathrm{PU}\text{-}{\mathrm{neg}}}(g) = \frac{1}{n_\mathrm{u}} \sum_{i=1}^{n_\mathrm{u}}\ell(-g(\boldsymbol{x}^\mathrm{u}_i)) - \frac{\pi}{n_\mathrm{p}}\sum_{i=1}^{n_\mathrm{p}}\ell(-g(\boldsymbol{x}^\mathrm{p}_i)).
\end{align*}

\subsection{Density ratio estimation}
In the density ratio estimation problem~\citep{sugiyama2012density}, we are given patterns from the two different distributions as follows: 
\begin{align*}
 X_\mathrm{nu}&:= \{\feature^\mathrm{nu}_i\}_{i=1}^{n_\mathrm{nu}} \stackrel{\mathrm{i.i.d.}}{\sim} p_\mathrm{nu}(\feature) \text{,}\\ 
 X_\mathrm{de}&:=  \{\feature^\mathrm{de}_i\}_{i=1}^{n_\mathrm{de}} \stackrel{\mathrm{i.i.d.}}{\sim} p_\mathrm{de}(\feature) \text{,}
\end{align*}
where $n_\mathrm{nu}$ and $n_\mathrm{de}$ are the numbers of patterns drawn from $p_\mathrm{nu}$ and $p_\mathrm{de}$, respectively (\dbquote{nu} indicates \dbquote{numerator} and \dbquote{du} indicates \dbquote{denominator}). The density ratio function is defined as follows:
\begin{align*}
w(\feature)=\frac{p_\mathrm{nu}(\feature)}{p_\mathrm{de}(\feature)} \text{.}
\end{align*}
 The goal of density ratio estimation is to find an accurate estimate $\widehat{w}$ of the density ratio function $w$ using patterns $X_\mathrm{nu}$ and $X_\mathrm{de}$~\citep{sugiyama2012density}. There are many existing methods that can be applied for estimating the density ratio function~\citep{kliep,ulsif,rulsif,trimratio,adityaratio}.

To the best of our knowledge, there is no experiment that applies density ratio estimation for PU classification. Nevertheless, density ratio estimation has been applied for inlier-based outlier detection~\citep{hido2011statistical} which is closely related with PU classification. It was pointed out that the important difference is that inlier-based outlier detection only produces an outlier score to evaluate the \emph{outlyingness} of patterns, while PU classification requires the threshold between positive and negative classes~\citep{du2015convex}. In this paper, we use a density ratio for PU classification under class prior shift and asymmetric error.

\section{PU classification under class prior shift}
In this section, we describe PU classification under class prior shift. The only difference from the ordinary PU classification discussed in the previous section is that the class prior of the test marginal density $\ptest$ and the unlabeled marginal density $\pu$ can be different. Formally, we can define $\pu$, $\ptest$, and related quantities as follows:
\begin{align*}
\pu(\feature) &:= \pi \ppos(\feature)  + (1-\pi)\pneg(\feature) \text{,} \\ 
\ptest(\feature) &:= \pitest\ppos(\feature) + (1-\pitest)\pneg(\feature)\text{,} 
\end{align*}
where $\ppos(\feature) := p(\feature|y=+1)$,  $\pneg(\feature) := p(\feature|y=-1)$, and $\pitest := \pi + \gamma$ for some $\gamma\in\R$ such that $0 < \pitest < 1$. If $\gamma = 0$, this problem is reduced to the ordinary PU classification where $\ptest=\pu$. In this setting, we have no access to the data in the test distribution. For this reason, it is \emph{impossible} to estimate $\pitest$. Thus, we assume that the test class prior $\pitest$ is given. In the experiment section, we also investigate the performance in the scenario where the given test class prior is incorrect. 

In binary classification, it is known that the Bayes optimal classifier, which is the classifier that achieves the minimum misclassification risk, can be written as follows~\citep{mohrifoundation}:
\begin{align} \label{bayes2}
f_{\mathrm{Bayes}}^*(\feature) =  \mathrm{sign}\left[p_{\pitest}(y=+1|\feature)-\frac{1}{2} \right],
\end{align}
where $p_{\pitest}(y=+1|\feature)$ denotes the class posterior probability which can be expressed by the Bayes' rule as follows: 
\begin{align*} 
	p_{\pitest}(y=+1|\feature) &=  \frac{\pitest \ppos(\feature)}{\ptest(\feature)}\text{.}
\end{align*}

\section{PU classification with asymmetric error}
In this section, we consider another extension of PU classification to the setting where the false positive penalty and false negative penalty are different (asymmetric error). To the best of our knowledge, while the setting of PU classification with the symmetric error is well-studied, the asymmetric error setting has not been studied yet. Let $\alpha$ be a false positive penalty ($0 < \alpha < 1$) and $(1-\alpha)$ be a false negative penalty.\footnote{Note that the range of $\alpha$ does not restrict the applicability because we can always normalize the false positive penalty and false negative penalty to be within this range.} If $\alpha=0.5$, the problem is reduced to the symmetric error setting. It is known that the Bayes optimal classifier of the binary classification with asymmetric error can be expressed as follows~\citep{scott2012calibrated}:
\begin{align} \label{bayes-asym}
f_{\mathrm{Bayes\text{-}}\alpha}^*(\feature) =  \mathrm{sign}\bigg[p_{\pi}(y=+1|\feature)-\alpha \bigg],
\end{align}
where $p_{\pi}(y=+1|\feature)$ denotes the class posterior probability, which can be expressed as 
\begin{align*} 
	p_{\pi}(y=+1|\feature) &=  \frac{\pi \ppos(\feature)}{\pu(\feature)}\text{.}
\end{align*}

In binary classification (including PU classification) with asymmetric error, the threshold to classify a pattern is shifted (depending on $\alpha$) but the class posterior probability remains the same as the symmetric error scenario.

\section{The equivalence of PU classification under class prior shift and PU classification with asymmetric error}
\begin{figure*}[t]
\includegraphics[width=\columnwidth]{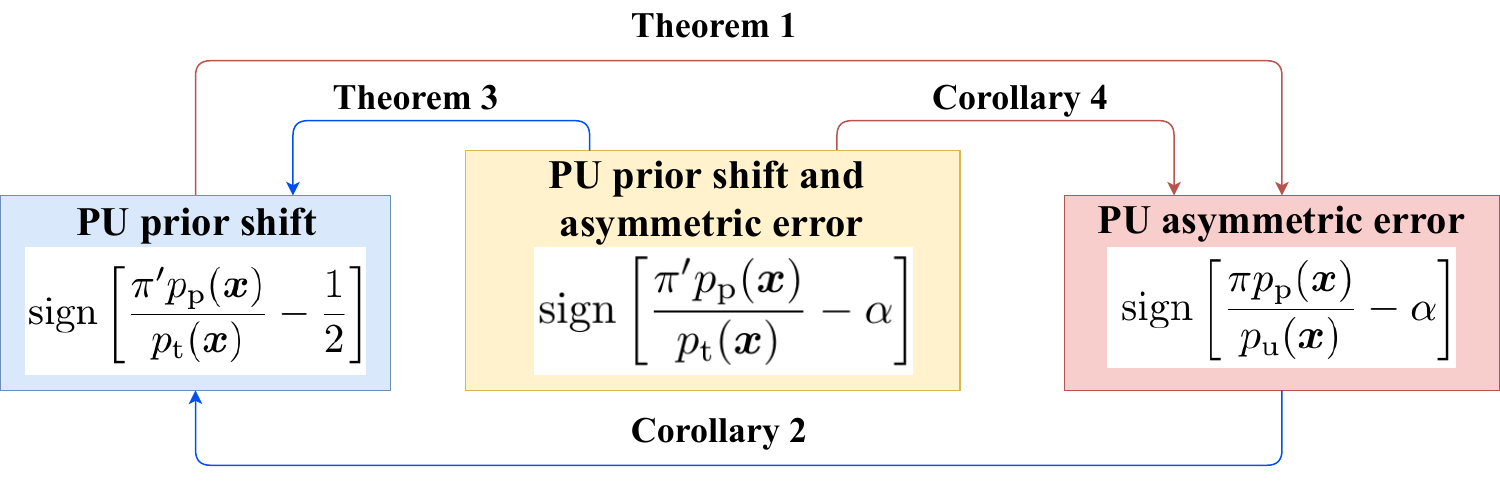}
\caption{The relationships between PU classification problems and the Bayes optimal classifiers for the problems.}
\label{fig:illustration}
\end{figure*}
In this section, we establish the connection between PU classification with asymmetric error and PU classification under class prior shift. Based on the Bayes optimal classifiers in Eq.~\eqref{bayes2} and Eq.~\eqref{bayes-asym}, the following theorem states that PU classification under class prior shift can be cast to PU classification with asymmetric error.
\begin{theorem} \label{col-shift-asym-alpha}Assume the test class prior $\pitest$ is given, the following equation holds:
\begin{align*}
\mathrm{sign} \left[\frac{\pi \ppos(\feature)}{\pu(\feature)}-\alpha\right] = \mathrm{sign} \left[ {\frac{\pitest \ppos(\feature)}{\ptest(\feature)} -\frac{1}{2}} \right] \text{,}
\end{align*}
where $\alpha$ can be expressed as
\begin{align}\label{shift-asym-alpha-corollary}
\alpha = \frac{\pi-\pi\pitest}{\pitest+\pi-2\pi\pitest}\text{.}
\end{align}
\begin{proof}
Consider the misclassification risk for PU classification with asymmetric error as follows \citep{scott2012calibrated}:
\begin{align*}
R^{\zerooneloss}_{\mathrm{asym}}(g) &= \pi (1-\alpha) \Epos \left[ \zerooneloss(g(\feature) \right] + (1-\pi)\alpha \Eneg \left[ \zerooneloss(-g(\feature)) \right] \text{.} 
\end{align*}

Next, let us consider another misclassification risk for PU classification under class prior shift:
\begin{align*}
R^{\zerooneloss}_{\mathrm{shift}}(g) = \pitest \Epos \left[ \zerooneloss(g(\feature) \right] + (1-\pitest) \Eneg \left[ \zerooneloss(-g(\feature)) \right] \text{.}
\end{align*}

Then, by normalizing $ \pi (1-\alpha) + (1-\pi)\alpha$ to one and equating the asymmetric error to the misclassification risk for PU classification under class prior shift, we have
\begin{align*}
\frac{\pi (1-\alpha)}{ \pi (1-\alpha) + (1-\pi)\alpha} = \pitest\text{.}
\end{align*}
Then, solving $\alpha$ based on the given training class prior~$\pi$ and the test class prior $\pitest$, we obtain the equivalent misclassification risk and therefore both scenarios have the same optimal Bayes classifier. 
\end{proof}
\end{theorem}
Similarly, we can also cast PU classification with asymmetric error to PU classification under class prior shift by using the following corollary obtained from Eq.~\eqref{shift-asym-alpha-corollary}.
\begin{corollary} \label{th-shift-asym-relation}Let $\alpha \in (0, 1)$ be a false positive penalty, and $1-\alpha$ be a false negative penalty. Then, the following equation holds:
\begin{align}\label{shift-asym-relation-2}
\mathrm{sign} \left[\frac{\pi \ppos(\feature)}{\pu(\feature)}-\alpha\right] = \mathrm{sign} \left[ {\frac{\pitest \ppos(\feature)}{\ptest(\feature)} -\frac{1}{2}} \right] \text{,}
\end{align}
where $\pitest$ can be expressed as
\begin{align*} 
\pitest =  \frac{\pi-\alpha \pi}{\pi+\alpha - 2\alpha \pi} \text{.}
\end{align*}
\end{corollary}
With our results, we can cast PU classification under class prior shift to PU classification with asymmetric error with Eq. \eqref{shift-asym-alpha-corollary} and vice versa with Eq. \eqref{shift-asym-relation-2}.
%

\section{PU classification under class prior shift and asymmetric error}
In this section, our target problem is PU classification where both class prior shift and asymmetric error are considered simultaneously. We show that the unified formulation can be given as follows: 
\begin{theorem}
Let $\alpha \in (0, 1)$ be a false positive penalty, $1-\alpha$ be a false negative penalty, and $\pitest=\pi+\gamma$ be a test class prior. Then, the following equation holds:
\begin{align}\label{shift-asym-unified}
\mathrm{sign} \left[\frac{\pitest \ppos(\feature)}{\ptest}-\alpha\right] = \mathrm{sign} \left[ {\frac{\piu \ppos(\feature)}{p_{\piu}(\feature)} -\frac{1}{2}} \right] \text{,}
\end{align}
where  $p_{\piu}(\feature) = \piu \ppos(\feature)+(1-\piu) \pneg(\feature)$, and $\piu$ is defined as 
\begin{align} \label{shift-asym-pi}
	\piu = \frac{\pitest - \alpha \pitest}{\pitest + \alpha - 2\alpha\pitest}\text{.}
\end{align}
\end{theorem}
With Eq. \eqref{shift-asym-unified}, we can cast the target problem to PU classification with symmetric error where the training class prior is $\pi$ and the test class prior is $\piu$. Similarly to the previous section, we can also cast our target problem to PU classification with asymmetric error (without class prior shift) using the following corollary:
\begin{corollary}
Let $\alpha \in (0, 1)$ be a false positive penalty, $1-\alpha$ be a false negative penalty, $\pitest=\pi+\gamma$ be a test class prior, and $\piu$ is defined in Eq. \eqref{shift-asym-pi}. Then, the following equation holds:
\begin{align*}
\mathrm{sign} \left[ {\frac{\pitest \ppos(\feature)}{\ptest(\feature)} -\alpha}\right]   = \mathrm{sign} \left[\frac{\pi \ppos(\feature)}{\pu}-\alphau\right]\text{,}
\end{align*}
where $\alphau$ is defined as 
\begin{align}\label{alphasilverbullet}
\alphau = \frac{\pi-\pi\piu}{\piu+\pi-2\pi\piu}\text{.}
\end{align}
\end{corollary}
With our results, one can freely choose whether to solve PU classification with class prior shift from $\pi$ to $\piu$ or PU classification with asymmetric error where the false positive penalty is $\alphau$ to handle PU classification under class prior shift and asymmetric error. The relationships between problems are illustrated in Figure~\ref{fig:illustration}.
\section{Algorithms for PU classification under class prior shift and asymmetric error}
In this section, we consider two different frameworks to handle PU classification under prior shift and asymmetric error. 
\subsection{Risk minimization framework} The goal of PU classification under class prior shift is to find a classifier~$g$ which minimizes the following misclassification risk for the test distribution:
\begin{align*}
R^{\zerooneloss}_{\mathrm{shift}}(g) = \pitest \Epos \left[ \zerooneloss(g(\feature)) \right] + (1-\pitest) \Eneg \left[ \zerooneloss(-g(\feature)) \right] \text{.}
\end{align*}
Our following theorem, which can be regarded as a special case of binary classification from unlabeled data~\citep{lu2018minimal}, states that we can express $R^{\ell}_{\mathrm{shift}}(g)$ using only the expectations of the positive and unlabeled data.
\begin{theorem} \label{pu-risk-shift}
The misclassification risk with a surrogate loss $R^{\ell}_{\mathrm{shift}}(g)$ can be equivalently expressed as
\begin{align*}
	R^{\ell}_{\mathrm{PU\text{-}shift}}(g) &= \Epos \left[ \pitest \ell(g(\feature)) - \frac{\pi-\pi\pitest}{1-\pi} \ell(-g(\feature)) \right] \numberthis \label{pu-expected-risk-shift} \\
	&\quad + \frac{1-\pitest}{1-\pi} \Eunl \left[ \ell(-g(\feature)) \right]. 
\end{align*}
\end{theorem}
\begin{proof}
First, we show that $(1-\pitest)\Eneg[\ell(-g(\feature))]$ can be rewritten in terms of the expectations of positive and unlabeled data. Based on the fact that $\Eunl[\ell(-g(\feature))] = \pi\Epos[\ell(-g(\feature))] + (1-\pi) \Eneg[\ell(-g(\feature))]$, we can express $\Eneg[\ell(-g(\feature))]$ as follows:
\begin{align*}
	R^{\ell}_{\mathrm{shift}}(g)  &= \pitest \Epos[\ell(g(\feature))] + (1-\pitest) \Eneg[\ell(-g(\feature))] \\
			 &= \pitest \Epos[\ell(g(\feature))] \\
			 &\quad + \frac{1-\pitest}{1-\pi}\left[  \Eunl[\ell(-g(\feature))] - \pi\Epos[\ell(-g(\feature))]   \right] \\
			 &=  \Epos [ \pitest \ell(g(\feature)) - \frac{\pi-\pi\pitest}{1-\pi} \ell(-g(\feature)) ] \\
	&\quad + \frac{1-\pitest}{1-\pi} \Eunl \left[ \ell(-g(\feature)) \right]  \\
	&= R^{\ell}_{\mathrm{PU\text{-}shift}}(g).
\end{align*}
Thus, we conclude that $R^{\ell}_{\mathrm{shift}}(g)  =  R^{\ell}_{\mathrm{PU\text{-}shift}}(g)$. 
\end{proof}

As a result, we can estimate $R^{\ell}_{\mathrm{PU\text{-}shift}}(g)$ from positive and unlabeled patterns. The unbiased estimator of the risk can be expressed as
\begin{align*}
	\widehat{R}^{\ell}_{\mathrm{PU\text{-}shift}}(g) &= \frac{1}{n_\mathrm{p}} \sum_{i=1}^{n_\mathrm{p}} \left[ \pitest \ell(g(\feature^\mathrm{p}_i)) - \frac{\pi-\pi\pitest}{1-\pi} \ell(-g(\feature^\mathrm{p}_i)) \right] \\
	&\quad + \frac{1}{n_\mathrm{u}}  \frac{1-\pitest}{1-\pi}  \sum_{i=1}^{n_\mathrm{u}}  \ell(-g(\feature^\mathrm{u}_i)). 
\end{align*}
This estimator is an unbiased estimator that the consistency and convergence rate of learning can be theoretically guaranteed~\citep{lu2018minimal}. Here, unlike the existing work which focused on learning from unlabeled data~\citep{lu2018minimal}, we are interested in PU classification under class prior shift and provide more results for the risk minimization framework approach for our problem setting. Similarly to the non-negative risk estimator for ordinary PU classification~\citep{kiryo2017}, we can obtain the non-negative risk estimator for PU classification under class prior shift follows:
\begin{align*}
	\widehat{R}^{\ell}_{\mathrm{nnPU}\text-\mathrm{Sh}}(g) &=  \frac{\pitest}{n_\mathrm{p}} \sum_{i=1}^{n_\mathrm{p}}\ell(g(\boldsymbol{x}^\mathrm{p}_i)) +  \frac{1-\pitest}{1-\pi} \max (0, \widehat{R}_{\mathrm{PU}\text{-}{\mathrm{neg}}}(g) ).
\end{align*}

Next, an interesting question is whether a convex formulation can still be obtained similarly to the ordinary PU classification \citep{du2015convex} for the risk $R^{\ell}_{\mathrm{PU\text{-}shift}}$ in Eq.~\eqref{pu-risk-shift}. Here, we show that convexity can be guaranteed \emph{when the training class prior is less than or equal to the class prior of the test data}, i.e., $\pi \leq \pitest$ ($\gamma \geq 0$).
\begin{theorem}
 Let $\ell$ be a convex loss such that $\ell(z)-\ell(-z)$ is linear and the training class prior is less than or equal to test class prior, i.e., $\pi \leq \pitest$ ($\gamma \geq 0$). With a linear-in-parameter model for $g$, $R^{\ell}_{\mathrm{PU\text{-}shift}}(g)$ is convex.
\end{theorem}

\begin{proof}
We prove using the following facts:
\begin{align*}
\ell(z) - \ell(-z) &= -z \text{,}\\
	\pitest &= \pi + \gamma \text{,}\\
    0 & < \pitest, \pi < 1 \text{,}\\
    -\pi & < \gamma < 1-\pi \text{.}
\end{align*}
For the analysis purpose, let us rewrite Eq. \eqref{pu-expected-risk-shift} as 
\begin{align*}
	R^{\ell}_{\mathrm{PU\text{-}shift}}(g) &= \pitest \Epos[\ell(g(\feature))] + (1-\pi-\gamma) \Eneg[\ell(-g(\feature))] \\
 &=   \pitest \Epos[\ell(g(\feature))] + \Eunl[\ell(-g(\feature))] \\ &\quad - \pi \Epos[\ell(-g(\feature))] - \gamma \Eneg[\ell(-g(\feature))] \\
 &= \Epos[\pitest \ell(g(\feature)) - \pi \ell(-g(\feature))] + \Eunl[\ell(-g(\feature))]\\ &\quad   -\frac{\gamma}{1-\pi} [\Eunl[\ell(-g(\feature))] - \pi \Epos[\ell(-g(\feature))]] \\
 &= -\pi \Epos[g(\feature)] + \Epos[\gamma \ell(g(\feature))+ \frac{\gamma \pi}{1-\pi}\ell(-g(\feature))] \\ &\quad   + (1-\frac{\gamma}{1-\pi})\Eunl[\ell(-g(\feature))].
\end{align*}
Then, we can express $R^{\ell}_{\mathrm{PU\text{-}shift}}(g)$ as follows:
\begin{align*}
R^{\ell}_{\mathrm{PU\text{-}shift}}(g) &=-\pi \Epos[g(\feature)] + \gamma \Epos[\ell(g(\feature))   + \frac{\pi}{1-\pi}\ell(-g(\feature))]  \numberthis \label{eq:troublesome} \\ 
&\quad + (1-\frac{\gamma}{1-\pi})\Eunl[\ell(-g(\feature))]  \text{.}  
\end{align*}
Note that we used the linear-odd condition $\Epos[\pi \ell(g(\feature)) - \pi \ell(-g(\feature))] = - \pi \Epos[g(\feature)]$ ~\citep{du2015convex}. On the right-hand side of Eq. \eqref{eq:troublesome}, the first term $ -\pi \Epos[g(\feature)] $ is linear. The second term is convex because $\gamma$ is positive and the terms inside the expectation operator are convex. The third term  $(1-\frac{\gamma}{1-\pi})\Eunl[\ell(-g(\feature))]$ is convex because $(1-\frac{\gamma}{1-\pi})$ is always non-negative and $\ell$ is convex. Thus, $R^{\ell}_{\mathrm{PU\text{-}shift}}(g)$ is convex since the sum of convex functions is convex~\citep{boyd2004convex}. This concludes the proof. 
\end{proof}

On the other hand, if $\pi > \pitest $, the problem lies in the second term of the right-hand side of Eq. \eqref{eq:troublesome} which is $\gamma \Epos[\ell(g(\feature))+ \frac{\pi}{1-\pi}\ell(-g(\feature))]$.  The terms inside the expectation are convex. However, the convex term is multiplied by $\gamma$. If $\pi > \pitest $, $\gamma$ is negative. This causes the second term of the right-hand side of Eq. \eqref{eq:troublesome} to be concave and the convexity of the whole formulation cannot be guaranteed.

Next, with the similar procedures, we can obtain similar results for PU classification with asymmetric error. The proofs of the following two theorems are given in Appendix.
\begin{theorem} Let $\alpha \in (0, 1)$ be a false positive error and $1-\alpha$ be a false negative error. The misclassification risk for the asymmetric error can be equivalently expressed as \label{th:appdx1}
\begin{align*}
	R^{\ell}_{\mathrm{PU\text{-}asym}}(g) &= \pi \Epos \left[ (1-\alpha) \ell(g(\feature)) - \alpha \ell(-g(\feature)) \right] \\	
		& \quad+ \alpha \Eunl \left[ \ell(-g(\feature)) \right] \text{.}
\end{align*}
\end{theorem} 
\begin{theorem}  \label{th:appdx2}
 Let $\ell$ be a convex loss such that $\ell(z)-\ell(-z)$ is linear and the false positive penalty is less than or equal to the false negative penalty, i.e., $\alpha \le 0.5$. With a linear-in-parameter model for $g$, $R^{\ell}_{\mathrm{PU\text{-}asym}}(g)$ is convex.
\end{theorem}

It can be observed that $\alpha$ determines the convexity of the formulation. If $\alpha \leq 0.5$, the convex formulation can be obtained. However, a convex formulation may not be obtained when  $\alpha>0.5$, i.e., we cannot guarantee the convexity of the formulation when the false positive penalty is higher than the false negative penalty.

\subsection{Density ratio framework}
Here, we propose a density ratio framework for PU classification under class prior shift and asymmetric error. First, we can reduce any problem in Figure~\ref{fig:illustration} to PU classification with asymmetric error $\alphau$ using Eq. \eqref{alphasilverbullet}. The Bayes optimal classifier for the asymmetric error scenario with $\alphau$ can be given as 
\begin{align} \label{eq:classtoratio}
f_{\mathrm{Bayes\text{-}}\alphau}^*(\feature) =  \mathrm{sign}\bigg[ \frac{\pi \ppos(\feature)}{\pu(\feature)}-\alphau \bigg].
\end{align}
Next, using the fact that we only consider the sign, another formulation can be obtained as follows:
\begin{align}\label{eq:ratiotodiff}
	 f_{\mathrm{Bayes\text{-}}\alphau}^*(\feature) =  \mathrm{sign}\left[\frac{\pi}{\alphau}  -\frac{\pu(\feature)}{\ppos(\feature)}\right] \text{,}
\end{align}
where the right-hand side of Eq. \eqref{eq:classtoratio} is multiplied by the positive value $\frac{\pu(\feature)}{ \ppos(\feature)\alphau}$.


After obtaining the training class prior $\pi$, PU classification based on density ratio estimation can be applied with the following steps: First, estimate $\frac{\ppos(\feature)}{\pu(\feature)}$. Second, calculate $\alphau$ using Eq. \eqref{alphasilverbullet}. Finally, we classify each test pattern by using Eq. \eqref{eq:classtoratio}. Another option is to estimate $\frac{\pu(\feature)}{\ppos(\feature)}$ and use Eq. \eqref{eq:ratiotodiff}. Clearly, the change of test class priors is not related to the step of calculating the density ratio but only the step of calculating $\alphau$. Therefore, the class prior shift and asymmetric error scenarios do not affect the convexity of the formulation. More precisely, the convexity of the density ratio framework entirely depends on how we estimate the density ratio. For example, by using a linear-in-parameter model, a convex formulation of density ratio estimation can be obtained via existing density ratio estimation methods such as the Kullback-Leibler importance estimation procedure~\citep{kliep} or unconstrained least-squares importance fitting (uLSIF)~\citep{ulsif}. 

Note that although Eq. \eqref{eq:classtoratio} and Eq. \eqref{eq:ratiotodiff} are equal if the estimated density ratio is perfectly accurate, the quality of estimators obtained in each formulation can be different in practice. One might argue that using the density ratio framework might not be effective because it is unstable since the density ratio can be unbounded and causes the estimation to be unreliable. However, in PU classification, our following simple proposition shows the boundedness of the density ratios in PU classification:
\begin{proposition} \label{proposition:dr}
$\frac{\ppos(\feature)}{\pu(\feature)}$ and $\frac{\pu(\feature)}{\ppos(\feature)}$ are both bounded as follows:
\begin{align*}
0 &< \frac{\ppos(\feature)}{\pu(\feature)} < \frac{1}{\pi} \text{,} \\
\pi &< \frac{\pu(\feature)}{\ppos(\feature)} \text{.}
\end{align*}
\end{proposition}

For $\frac{\ppos(\feature)}{\pu(\feature)}$, the denominator contains all the supports from the numerator. The density ratio $\frac{\ppos(\feature)}{\pu(\feature)}$ is upper-bounded by $\frac{1}{\pi}$ since $p(y=+1|\feature) = \pi \frac{\ppos(\feature)}{\pu(\feature)}$ and $p(y=+1|\feature) \in [0,1]$. On the other hand, $\frac{\pu(\feature)}{\ppos(\feature)}$ is  unbounded. Thus, Proposition \ref{proposition:dr} suggests that density ratio estimation in PU classification with $\frac{\ppos(\feature)}{\pu(\feature)}$ can potentially be more suitable. We illustrate the advantage of using the density ratio $\frac{\ppos(\feature)}{\pu(\feature)}$ in the experiment section. 

In addition, one can determine the density ratio of the test distribution $\frac{\ptest(\feature)}{\ppos(\feature)}$ from the training data based on the following theorem for the class prior shift scenario.
\begin{theorem} \label{th:drshift}
Let $\pitest=\pi+\gamma$ be a test class prior, $\frac{\ptest(\feature)}{\ppos(\feature)}$  can be expressed as
\begin{align*}
\frac{\ptest(\feature)}{\ppos(\feature)} =  \frac{\gamma}{1-\pi} + (1-\frac{\gamma}{1-\pi}) \frac{\pu(\feature)}{\ppos(\feature)} \text{.} 
\end{align*}
\end{theorem}

The proof is given in Appendix. We can see that $\frac{\ptest(\feature)}{\ppos(\feature)}$ can be obtained with an affine transformation of $\frac{\pu(\feature)}{\ppos(\feature)}$. Thus, we can obtain $\frac{\ptest(\feature)}{\ppos(\feature)}$ by estimating  $\frac{\pu(\feature)}{\ppos(\feature)}$ given that we can access unlabeled patterns and positive patterns. 

\section{Comparison between a risk minimization framework and density ratio framework}
In this section, we compare our two proposed frameworks in terms of the convexity of the formulation and the necessity to retrain a classifier when the test condition changes. Without loss of generality, we cast PU classification under class prior shift and asymmetric error to PU classification with asymmetric error with $\alphau$ using Eq.~\eqref{alphasilverbullet}. 

\noindent \textbf{Convexity using a linear-in-parameter model:}
The convexity of the risk minimization framework depends on $\alphau$. The convex formulation can be obtained if $\alphau \leq 0.5$ using a convex loss that satisfies the linear-odd condition according to Theorem \ref{th:appdx2}. On the other hand, irrespective of $\alphau$, the density ratio framework always yields a convex formulation with convex density ratio estimation methods such as uLSIF~\citep{ulsif}.

\noindent \textbf{Retraining when the test condition changes:}
The risk minimization framework requires to minimize the surrogate risk with respect to the new class prior every time when the test condition changes (class prior shift or asymmetric error). On the other hand, we do not need to recalculate the density ratio when the test condition changes. We simply need to adjust the parameter $\alphau$, which can be calculated by Eq.~\eqref{alphasilverbullet}. Thus, from the viewpoint of the need to retrain the classifier, the density ratio framework is more suitable in the situation where the test conditions frequently change.

\section{Experimental results}
In this section, we present the experimental results of PU classification under class prior shift using benchmark datasets. 

\noindent \textbf{Implementation:} For the density ratio framework, we used unconstrained least-squares important fitting (uLSIF)~\citep{ulsif} with Gaussian kernel bases. 
We implemented density ratio estimation based on  Eq.~\eqref{eq:classtoratio} by computing the ratio $\frac{\ppos(\feature)}{\pu(\feature)}$ ($\frac{p}{u}\text{uLSIF}$) and Eq.~\eqref{eq:ratiotodiff} by computing the ratio $\frac{\pu(\feature)}{\ppos(\feature)}$ ($\frac{u}{p}\text{uLSIF}$), respectively. 
For the risk minimization framework, we used the squared loss (Sq) and double hinge loss (DH)~\citep{du2015convex}. 
We implemented the linear-in-parameter model using input features directly~(Lin) and using Gaussian kernel bases (Ker), respectively. For fairness between the kernel model and linear model using input features directly, we used the same number of hyperparameter candidates for each algorithm. The risk minimization objectives were optimized using AMSGRAD~\citep{adam_amsgrad} and the experiment code was implemented using Chainer~\citep{tokui2015chainer}.

\begin{table}[t]
\centering
\caption{Mean accuracy and standard error over 10 trials for PU classification,  where  $\pi = 0.3$ and $\pi^*=0.5$. Outperforming methods are highlighted in boldface using one-sided t-test with the significance level 5\%.}
\scalebox{0.65}{\begin{tabular} { |L|L|L|L|L|L|L|L|L|L|L|L|L|L|}
\hline
\text{Dataset} & \pi_{\mathrm{g}}^*& \frac{u}{p}\text{uLSIF}& \frac{p}{u}\text{uLSIF}& \text{DH-Lin}& \text{DH-Ker}& \text{Sq-Lin}& \text{Sq-Ker}\\ \hline
\text{banana} &  &85.2 (0.4) & \textbf{87.2 (0.7)} &52.2 (1.4) &83.6 (0.9) &53.1 (1.4) & \textbf{86.2 (0.6)}\\ 
\text{ijcnn1} &  &67.6 (0.8) & \textbf{70.9 (0.9)} &67.6 (1.1) &52.3 (0.7) & \textbf{73.0 (1.0)} &57.4 (0.6)\\ 
\text{MNIST} & \pi^* &71.8 (0.4) &83.6 (0.5) &83.0 (0.9) & \textbf{84.7 (0.6)} &69.1 (0.9) & \textbf{86.1 (0.7)}\\ 
\text{susy} & 0.5&71.2 (0.4) & \textbf{73.9 (0.5)} &70.7 (0.4) &65.7 (0.7) & \textbf{75.2 (0.7)} &69.0 (0.6)\\ 
\text{cod-rna} &  &84.6 (0.6) &84.0 (0.5) & \textbf{85.9 (0.4)} &80.6 (0.8) & \textbf{86.1 (0.4)} &83.8 (0.6)\\ 
\text{magic} &  &63.1 (0.7) &73.0 (0.5) & \textbf{76.5 (0.7)} &67.3 (1.0) & \textbf{76.6 (0.8)} &73.3 (0.7)\\ 
\hline
\text{banana} & &84.2 (0.4) & \textbf{88.3 (0.6)} &52.7 (0.6) &75.8 (1.6) &50.6 (0.3) &84.7 (0.8)\\ 
\text{ijcnn1} &  &58.1 (0.4) & \textbf{65.4 (0.6)} &50.4 (0.1) &50.0 (0.0) &61.0 (0.4) &50.6 (0.2)\\ 
\text{MNIST} &  &70.2 (0.6) &83.5 (0.5) &82.3 (1.0) &82.3 (0.4) &69.6 (0.8) & \textbf{85.1 (0.5)}\\ 
\text{susy} & 0.4 &69.1 (0.5) & \textbf{73.4 (0.4)} &65.7 (0.5) &52.2 (0.6) &71.8 (0.5) &66.7 (0.8)\\ 
\text{cod-rna} & &83.4 (0.6) &84.1 (0.6) & \textbf{86.0 (0.3)} &80.6 (1.1) & \textbf{86.3 (0.4)} &83.9 (0.6)\\ 
\text{magic} &  &61.8 (0.6) &73.5 (0.7) & \textbf{76.2 (0.6)} &61.7 (0.9) & \textbf{76.7 (0.8)} &70.9 (0.7)\\ 
\hline
\text{banana} & &82.6 (0.5) & \textbf{86.3 (0.7)} &52.5 (0.5) &65.8 (2.0) &50.0 (0.0) &79.9 (1.2)\\ 
\text{ijcnn1} &  & \textbf{54.3 (0.3)} & \textbf{54.9 (0.5)} &50.1 (0.0) &50.0 (0.0) &53.4 (0.3) &50.0 (0.0)\\ 
\text{MNIST} & \pi &68.1 (0.6) &78.8 (0.5) & \textbf{82.0 (0.9)} &71.8 (0.8) &70.6 (0.6) & \textbf{80.7 (0.5)}\\ 
\text{susy} & 0.3 & \textbf{67.3 (0.8)} & \textbf{68.6 (0.7)} &54.8 (0.6) &50.2 (0.1) &66.5 (0.6) &60.0 (0.5)\\ 
\text{cod-rna} &  &80.0 (0.9) &82.5 (0.6) & \textbf{84.8 (0.3)} &70.5 (2.0) & \textbf{84.2 (0.4)} &81.4 (0.7)\\ 
\text{magic} &  &58.3 (0.6) &69.1 (0.7) & \textbf{72.9 (0.4)} &50.0 (0.0) &68.7 (0.8) &60.8 (0.8)\\ 
\hline
\end{tabular}}
\label{table:result1}
\end{table}

\begin{table}[t]
\centering
\caption{Mean accuracy and standard error over 10 trials for PU classification, where  $\pi = 0.7$ and $\pi^*=0.3$. Outperforming methods are highlighted in boldface using one-sided t-test with the significance level 5\%.}
\scalebox{0.65}{\begin{tabular} { |L|L|L|L|L|L|L|L|L|L|L|L|L|L|}
\hline
\text{Dataset} & \pi_{\mathrm{g}}^*& \frac{u}{p}\text{uLSIF}& \frac{p}{u}\text{uLSIF}& \text{DH-Lin}& \text{DH-Ker}& \text{Sq-Lin}& \text{Sq-Ker}\\ \hline
\text{banana} &  &83.0 (1.0) & \textbf{86.4 (0.5)} &70.2 (0.5) &78.3 (1.0) &70.0 (0.0) &83.4 (0.4)\\ 
\text{ijcnn1} &  &70.8 (0.6) & \textbf{74.2 (0.7)} &70.0 (0.1) &69.8 (0.2) &71.5 (0.3) &69.2 (0.5)\\ 
\text{MNIST} & \pi^* &79.3 (0.5) & \textbf{81.7 (0.5)} &74.0 (1.1) & \textbf{82.4 (1.0)} &52.3 (1.4) & \textbf{83.4 (0.9)}\\ 
\text{susy} & 0.3 &74.3 (0.5) & \textbf{76.0 (0.3)} &72.7 (0.6) &70.0 (0.0) & \textbf{75.5 (1.4)} &74.7 (0.7)\\ 
\text{cod-rna} &  &82.1 (1.0) &82.8 (0.8) & \textbf{87.3 (0.7)} &77.3 (0.8) & \textbf{85.2 (1.1)} &80.2 (1.0)\\ 
\text{magic} &  &71.5 (0.7) & \textbf{75.8 (0.6)} &72.7 (1.1) &70.8 (0.4) & \textbf{75.0 (1.0)} &72.9 (0.7)\\ 
\hline
\text{banana} & &84.7 (1.1) & \textbf{88.7 (0.7)} &54.9 (1.4) &81.7 (1.6) &53.6 (1.2) &83.8 (1.3)\\ 
\text{ijcnn1} &  & \textbf{64.9 (1.4)} & \textbf{66.6 (1.0)} &60.4 (1.4) &51.6 (3.0) &62.2 (1.2) &48.2 (2.8)\\ 
\text{MNIST} & 0.5 &81.9 (0.4) & \textbf{84.1 (0.6)} &72.5 (1.0) &82.5 (0.7) &52.9 (1.1) &81.9 (0.9)\\ 
\text{susy} &  & \textbf{75.9 (1.1)} & \textbf{77.0 (0.6)} &67.5 (1.4) &75.5 (0.6) &71.6 (1.0) &72.8 (1.1)\\ 
\text{cod-rna} &  & \textbf{85.3 (0.7)} & \textbf{85.4 (0.5)} & \textbf{86.2 (0.7)} &80.1 (1.1) & \textbf{86.5 (0.9)} &81.2 (1.2)\\ 
\text{magic} &  &67.6 (0.8) & \textbf{73.6 (0.9)} & \textbf{72.6 (0.7)} &62.4 (1.9) & \textbf{71.8 (0.7)} &68.9 (0.8)\\ 
\hline
\text{banana} &  & \textbf{80.6 (1.3)} & \textbf{82.1 (1.1)} &31.8 (0.9) &48.9 (1.5) &30.0 (0.0) &69.9 (1.1)\\ 
\text{ijcnn1} &  &35.2 (1.4) & \textbf{42.4 (0.9)} &30.0 (0.0) &30.0 (0.0) &32.4 (0.5) &30.9 (0.4)\\ 
\text{MNIST} & \pi & \textbf{79.9 (0.7)} &72.6 (0.6) &71.1 (1.1) &64.8 (1.1) &64.0 (0.6) &74.2 (1.0)\\ 
\text{susy} & 0.7 &35.6 (3.1) & \textbf{44.2 (2.9)} &30.0 (0.0) &30.0 (0.0) & \textbf{42.0 (1.5)} &36.8 (1.3)\\ 
\text{cod-rna} &  & \textbf{77.7 (2.2)} & \textbf{77.8 (2.1)} & \textbf{79.6 (0.7)} &67.8 (0.8) &78.2 (0.5) &68.3 (1.0)\\ 
\text{magic} &  &51.6 (0.3) & \textbf{60.3 (1.5)} & \textbf{56.2 (2.7)} &32.8 (0.7) & \textbf{58.7 (1.4)} &50.1 (1.6)\\ 
\hline
\end{tabular}}
\label{table:result2}
\end{table}
In summary, our proposed methods successfully improved the accuracy of PU classification by taking the difference between the training and test class priors into account. Moreover, the density ratio framework based on Eq.~\eqref{eq:classtoratio} is observed to be more robust when the given test class prior is incorrect, which implies the usefulness in PU classification under class prior shift. Sq-Lin and DH-Lin performed well but they cannot be used for the data that is difficult to separate with a linear hyperplane such as the \emph{banana} dataset. 
\begin{table}[t]
\centering
\caption{Mean accuracy and standard error over 10 trials for PU classification \emph{without} class prior shift and the correct class priors were given. Outperforming methods are highlighted in boldface using one-sided t-test with the significance level~5\%.}
\scalebox{0.65}{\begin{tabular} { |L|L|L|L|L|L|L|L|L|L|L|L|L|L|}
\hline
\text{Dataset} & \pi & \frac{u}{p}\text{uLSIF}& \frac{p}{u}\text{uLSIF}& \text{DH-Lin}& \text{DH-Ker}& \text{Sq-Lin}& \text{Sq-Ker}\\ \hline
\text{banana} &  &87.9 (0.5) & \textbf{90.1 (0.6)} &69.9 (0.5) &79.2 (1.1) &70.0 (0.0) &86.9 (0.7)\\ 
\text{ijcnn1} &  &71.6 (0.2) & \textbf{72.9 (0.4)} &70.1 (0.0) &70.0 (0.0) &71.7 (0.2) &70.0 (0.0)\\ 
\text{MNIST} & 0.3 &77.9 (0.6) & \textbf{86.0 (0.4)} &83.0 (0.8) &82.6 (0.6) &71.6 (0.7) & \textbf{86.8 (0.3)}\\ 
\text{susy} &  & \textbf{78.4 (0.5)} & \textbf{79.5 (0.5)} &73.0 (0.4) &70.2 (0.1) & \textbf{79.2 (0.3)} &75.0 (0.3)\\ 
\text{cod-rna} & &85.9 (0.7) &87.4 (0.6) & \textbf{88.9 (0.4)} &80.6 (1.3) & \textbf{88.6 (0.4)} &86.6 (0.7)\\ 
\text{magic} &  &70.3 (0.2) &76.7 (0.5) & \textbf{78.1 (0.4)} &70.0 (0.0) & \textbf{78.2 (0.5)} &74.1 (0.5)\\ 
\hline
\text{banana} &  &86.8 (0.7) & \textbf{89.3 (0.6)} &61.1 (0.4) &81.9 (1.2) &60.8 (0.3) &87.3 (0.5)\\ 
\text{ijcnn1} &  &66.0 (0.4) & \textbf{71.3 (0.5)} &60.1 (0.1) &60.0 (0.0) &67.8 (0.7) &59.9 (0.3)\\ 
\text{MNIST} & 0.4 &74.2 (0.8) &83.3 (0.5) &81.5 (0.6) &84.1 (0.4) &69.2 (0.5) & \textbf{86.0 (0.4)}\\ 
\text{susy} &  &73.7 (0.6) & \textbf{75.5 (0.5)} &71.6 (0.5) &62.0 (0.6) & \textbf{74.8 (0.3)} &73.0 (0.4)\\ 
\text{cod-rna} &  &84.9 (0.3) &85.8 (0.3) & \textbf{87.4 (0.4)} &80.3 (1.1) & \textbf{88.0 (0.5)} &84.6 (0.5)\\ 
\text{magic} &  &68.1 (0.6) &75.0 (0.6) &76.2 (0.4) &67.5 (0.8) & \textbf{77.7 (0.5)} &73.0 (0.6)\\ 
\hline
\text{banana} &  &85.2 (0.7) & \textbf{87.6 (0.5)} &51.6 (1.4) &83.0 (0.9) &51.1 (1.6) &85.2 (0.7)\\ 
\text{ijcnn1} &  &65.5 (0.6) & \textbf{68.4 (1.1)} &66.5 (0.9) &50.4 (0.5) & \textbf{70.6 (1.3)} &56.0 (0.7)\\ 
\text{MNIST} & 0.5 &73.4 (0.9) &83.9 (0.5) &81.5 (0.6) & \textbf{84.3 (0.7)} &69.7 (0.6) & \textbf{85.8 (0.6)}\\ 
\text{susy} &  &71.2 (0.5) & \textbf{73.5 (0.4)} &70.1 (0.6) &66.3 (0.7) & \textbf{73.1 (0.6)} &70.7 (0.7)\\ 
\text{cod-rna} &  &84.8 (0.6) &84.8 (0.5) & \textbf{86.2 (0.5)} &78.8 (1.3) & \textbf{86.3 (0.4)} &81.7 (1.2)\\ 
\text{magic} &  &64.1 (1.0) & \textbf{74.2 (0.8)} & \textbf{74.0 (0.8)} &68.5 (1.0) & \textbf{74.1 (0.7)} &71.7 (0.6)\\ 
\hline
\text{banana} &  &86.2 (0.6) & \textbf{88.6 (0.4)} &68.1 (1.1) &77.4 (0.6) &70.0 (0.0) &85.5 (0.7)\\ 
\text{ijcnn1} &  &69.8 (0.1) & \textbf{71.2 (0.5)} &70.0 (0.0) &70.0 (0.0) & \textbf{70.9 (0.2)} &68.9 (0.5)\\ 
\text{MNIST} & 0.7 &81.0 (0.8) & \textbf{84.6 (0.6)} &81.4 (0.6) &83.2 (0.4) &68.3 (0.7) & \textbf{85.5 (0.6)}\\ 
\text{susy} &  & \textbf{71.6 (0.6)} & \textbf{72.8 (0.9)} &70.0 (0.0) &70.0 (0.0) & \textbf{71.8 (0.6)} & \textbf{71.5 (0.3)}\\ 
\text{cod-rna} &  &83.6 (0.8) &83.0 (0.8) & \textbf{85.9 (0.5)} &78.9 (0.7) & \textbf{86.3 (0.5)} &78.7 (0.7)\\ 
\text{magic} &  &76.0 (0.5) & \textbf{78.2 (0.6)} & \textbf{77.6 (0.8)} &71.2 (0.3) & \textbf{78.4 (0.7)} &76.5 (0.6)\\ 
\hline
\end{tabular}}
\label{table:result3}
\end{table}

\noindent \textbf{Datasets:} For each dataset, training data composed of 500 positive patterns and 2,000 unlabeled patterns were used.  
The test data composed of 500 patterns was used and the proportion of positive and negative patterns depends on the test class prior. 
The datasets we used can be found in LIBSVM~\citep{dataset2} and UCI Machine Learning Repository~\citep{dataset3}. 
For the MNIST dataset, we used even digits as positive patterns and odd digits as negative patterns. 
Two different pairs of the training class prior and test class prior are considered: ($\pi$, $\pitest$) = ($0.3$, $0.5$) and ($\pi$, $\pitest$) = ($0.7$, $0.3$). 
We also investigated the performance when the given test class prior is incorrect. We define $\pig$ as the given test class prior (which may not be equal to $\pitest$). 
Additional experimental results can be found in Appendix for different class priors. 
If $\pig = \pi$, the method becomes the ordinary PU classification method but we evaluated the performance in the test environment where the test class prior is $\pitest$. 
Moreover, we also assume that the training class prior is given in order to illustrate the effect of the incorrect test class prior more precisely. In practice, one can estimate the training class prior by using existing class prior estimation methods~\citep{prior3, prior2}. 

\noindent \textbf{Discussion:} 
Tables~\ref{table:result1} and \ref{table:result2} show the results of PU classification under class prior shift with varying class priors. Table~\ref{table:result3} shows the results of ordinary PU classification. We can interpret the experimental results as follows: First, the performance of all algorithms can be improved by taking the differing test class priors into account which is the main objective of this paper. 
Second, we can see that $\frac{p}{u}\text{uLSIF}$ significantly outperformed $\frac{u}{p}\text{uLSIF}$ in almost all cases. This illustrates the advantage of estimating $\frac{\ppos(\feature)}{\pu(\feature)}$ over $\frac{\pu(\feature)}{\ppos(\feature)}$ due to the boundedness as suggested in Proposition~\ref{proposition:dr}. 
Third, $\frac{p}{u}\text{uLSIF}$ can be observed to be more robust than the risk minimization framework when the given test class prior is incorrect. 
Fourth, the Gaussian kernel-based linear-in-parameter model for the risk minimization framework (Sq-Ker, DH-Ker) did not work well in our experiments compared with other methods. In addition, we also tried using a linear-in-parameter model with input features directly for the density ratio framework. We found that the density ratio method based on this model failed miserably, while Gaussian kernels substantially improved the performance.

\section{Conclusions}
We investigated the binary classification problem from positive and unlabeled data (PU classification) under class prior shift and asymmetric error. We proved the equivalence of PU classification under class prior shift, PU classification with asymmetric error, and PU classification where both class prior shift and asymmetric error occur simultaneously. To handle such scenarios, we considered a risk minimization framework and density ratio framework. We provided the analysis of convexity and the comparison of the two frameworks. The experiments illustrated that our proposed frameworks successfully improved the accuracy of PU classification under class prior shift scenario.
\section{Acknowledgements}
We thank Han Bao and Seiichi Kuroki for helpful discussion. NC was supported by MEXT scholarship and MS was supported by JST CREST JPMJCR1403.


\newpage
\appendix
\allowdisplaybreaks
\section{Proof of Theorem~\ref{th:appdx1}}
\begin{proof}
The existing work on binary classification with asymmetric error \citep{scott2012calibrated} showed that if a loss $\ell$ is classification calibrated \citep{bartlett2006}, we can use a surrogate loss to handle the binary classification with asymmetric error. More specifically, we can use a classification-calibrated loss to minimize the following risk: 
\begin{align*}
R^{\ell}_{\mathrm{asym}}(g) &= \pi (1-\alpha) \Epos \left[ \ell(g(\feature) \right] \\
&\quad + (1-\pi)\alpha \Eneg \left[ \ell(-g(\feature)) \right]  \text{.} 
\end{align*}

Next, similar to the class prior shift, we show that $(1-\pi)\alpha\Eneg[\ell(-g(\feature))]$ can be rewritten to consist of the expectation of positive data and unlabeled data. Based on the fact that $\Eunl[\ell(-g(\feature))] = \pi\Epos[\ell(-g(\feature))] + (1-\pi) \Eneg[\ell(-g(\feature))]$, we can express $\Eneg[\ell(-g(\feature))]$ as follows:
\begin{align*}
	  \Eneg[\ell(-g(\feature))] &= \frac{ \Eunl[\ell(-g(\feature))] - \pi\Epos[\ell(-g(\feature))]}{1-\pi} \text{.}
\end{align*}

Then, by plugging in the rewritten version of $ \Eneg[\ell(-g(\feature))]$, we have
\begin{align*}
R^{\ell}_{\mathrm{asym}}(g) &= \pi (1-\alpha) \Epos \left[ \ell(g(\feature) \right] \\
&\quad + \alpha\Eunl[\ell(-g(\feature))] - \alpha \pi\Epos[\ell(-g(\feature))]  \text{.} \\
\end{align*}
Then, we can obtain
\begin{align*}
	R^{\ell}_{\mathrm{PU\text{-}asym}}(g) &= \pi \Epos \left[ (1-\alpha) \ell(g(\feature)) - \alpha \ell(-g(\feature)) \right] \\	
		& \quad+ \alpha \Eunl \left[ \ell(-g(\feature)) \right] \text{.}
\end{align*}
Thus, we conclude the proof. 
\end{proof}
\section{Proof of Theorem~\ref{th:appdx2}}
\begin{proof}
Similarly to the class prior shift scenario, we can rewrite the result in Theorem \ref{th:appdx1} to analyze the convexity of the problem when using a linear-in-parameter model and a convex loss that satisfies the linear-odd condition. Let $\alpha = 0.5 + \beta$, we can rewrite the risk of PU classification with asymmetric error as follows:
\begin{align*}\label{purisk-asym-analyze}
	R^{\ell}_{\mathrm{PU\text{-}asym}}(g) &=\frac{\pi}{2} \Epos[\ell(g(\feature)) - \ell(-g(\feature))] + \alpha \Eunl[\ell(-g(\feature))] \numberthis\\
	     &\quad - \beta\pi \Epos[\ell(g(\feature))+\ell(-g(\feature))] \text{.} 
\end{align*}

In Eq. \eqref{purisk-asym-analyze}, the choice of $\alpha$ can determine the convexity condition of the formulation. if $\alpha \leq 0.5$ ($\beta \leq 0$) , the convex formulation can be obtained while a convex formulation may not be guaranteed when  $\alpha>0.5$ ($\beta > 0$). More precisely, we cannot guarantee the convexity of the formulation of the task where the false positive error is higher than a false negative error. 
\end{proof}
\section{Proof of Theorem~\ref{th:drshift}}
\begin{proof} We can derive the result in the following steps:
\begin{align*}
	\frac{\ptest(\feature)}{\ppos(\feature)} &= \frac{(\pi+\gamma)\ppos(\feature) + (1-\pi-\gamma) \pneg(\feature)}{\ppos(\feature)} \\
    &= (\pi+\gamma)\frac{\ppos(\feature)}{\ppos(\feature)} + (1-\pi-\gamma)\frac{\pneg(\feature)}{\ppos(\feature)}\\
    &= \pi+\gamma+ (1-\pi)\frac{\pneg(\feature)}{\ppos(\feature)} - \gamma\frac{\pneg(\feature)}{\ppos(\feature)} \\
    &= \pi + \gamma+ \frac{\pu(\feature)}{\ppos(\feature)} - \pi \frac{\ppos(\feature)}{\ppos(\feature)} \\
    & \quad - \frac{\gamma}{1-\pi} ( \frac{\pu(\feature)}{\ppos(\feature)} - \frac{\pi \ppos(\feature)}{\ppos(\feature)}) \\
    &= \pi + \gamma+ \frac{\gamma \pi}{1-\pi} - \pi + \frac{\pu(\feature)}{\ppos(\feature)} \\
    &\quad - \frac{\gamma}{1-\pi}\frac{\pu(\feature)}{\ppos(\feature)} \\
    &= \gamma+ \frac{\gamma \pi}{1-\pi} + (1-\frac{\gamma}{1-\pi}) \frac{\pu(\feature)}{\ppos(\feature)} \\
    &= \frac{\gamma}{1-\pi} + (1-\frac{\gamma}{1-\pi}) \frac{\pu(\feature)}{\ppos(\feature)} \text{.}
\end{align*} 
Thus, $\frac{\ptest(\feature)}{\ppos(\feature)}$ can be obtained with an affine transformation of the density ratio of the training unlabeled data and positive data $\frac{\pu(\feature)}{\ppos(\feature)}$. 
\end{proof}

\section{Additional experimental results}
In this section, we present the experimental results for the different class priors from the main paper. The experiment settings are identical to the experiment section. Four additional different pairs of the training class prior and test class prior were considered: ($\pi$, $\pitest$) = \{($0.5$, $0.3$), ($0.3$, $0.7$), ($0.4$, $0.8$), ($0.8$, $0.4$)\}. Tables 4-7 indicate the mean accuracy and standard error of PU classification under class prior shift. 

\begin{table*}
\centering
\caption{Mean accuracy and standard error over 10 trials for PU classification, where  $\pi = 0.5$ and $\pitest=0.3$. Outperforming methods are highlighted in boldface using one-sided t-test with the significance level 5\%.}
\scalebox{0.8}{\begin{tabular} { |L|L|L|L|L|L|L|L|L|L|L|L|L|L|}
\hline
\text{Dataset} & \pig& \frac{u}{p}\text{uLSIF}& \frac{p}{u}\text{uLSIF}& \text{DH-Lin}& \text{DH-Ker}& \text{Sq-Lin}& \text{Sq-Ker}\\ \hline
\text{banana} &  & \textbf{87.4 (0.8)} & \textbf{88.2 (0.5)} &70.6 (0.5) &79.3 (1.2) &70.0 (0.0) &87.0 (0.4)\\ 
\text{ijcnn1} &  &72.2 (0.3) & \textbf{74.3 (0.5)} &70.1 (0.0) &70.0 (0.0) &71.3 (0.3) &70.1 (0.1)\\ 
\text{MNIST} & \pitest &78.2 (0.7) &84.5 (0.5) &80.4 (0.6) &83.3 (0.5) &55.7 (0.7) & \textbf{85.8 (0.4)}\\ 
\text{susy} & 0.3 & \textbf{78.8 (0.3)} & \textbf{78.3 (0.6)} &72.6 (0.6) &70.0 (0.0) & \textbf{78.7 (0.7)} &75.7 (0.6)\\ 
\text{cod-rna} &  &85.2 (0.7) & \textbf{87.4 (0.7)} & \textbf{88.5 (0.4)} &82.5 (1.3) & \textbf{87.8 (0.4)} &85.1 (0.9)\\ 
\text{magic} &  &70.7 (0.4) & \textbf{75.9 (0.4)} & \textbf{75.4 (0.9)} &70.1 (0.1) & \textbf{77.0 (0.6)} &74.9 (0.7)\\ 
\hline
\text{banana} &  &87.8 (0.7) & \textbf{89.8 (0.5)} &69.4 (1.0) &85.8 (0.6) &70.2 (0.5) & \textbf{88.3 (0.8)}\\ 
\text{ijcnn1} &  & \textbf{73.3 (0.6)} & \textbf{74.9 (0.8)} &70.0 (0.0) &70.0 (0.0) & \textbf{74.0 (0.5)} &69.3 (0.4)\\ 
\text{MNIST} &  &78.3 (1.2) & \textbf{86.3 (0.6)} &80.2 (0.6) & \textbf{86.2 (0.4)} &62.4 (0.7) & \textbf{86.4 (0.4)}\\ 
\text{susy} & 0.4 & \textbf{78.9 (0.4)} & \textbf{79.0 (0.7)} & \textbf{78.0 (0.6)} &71.9 (0.5) & \textbf{79.3 (0.6)} &76.5 (0.4)\\ 
\text{cod-rna} &  &87.0 (0.5) &87.5 (0.4) & \textbf{88.7 (0.5)} &80.3 (1.6) & \textbf{88.5 (0.4)} &83.7 (1.8)\\ 
\text{magic} &  &70.1 (0.5) & \textbf{76.0 (0.9)} & \textbf{74.7 (0.8)} &72.9 (0.9) & \textbf{76.1 (0.7)} & \textbf{74.8 (0.4)}\\ 
\hline
\text{banana} &  & \textbf{87.9 (0.8)} & \textbf{88.8 (0.5)} &51.9 (1.4) &83.0 (1.8) &51.3 (1.5) &85.6 (0.7)\\ 
\text{ijcnn1} &  & \textbf{73.4 (0.9)} &69.9 (1.1) &62.7 (1.6) &53.9 (4.2) & \textbf{71.3 (1.5)} &57.0 (1.4)\\ 
\text{MNIST} & \pi &77.7 (1.8) & \textbf{84.5 (0.6)} &79.9 (0.6) & \textbf{84.5 (0.6)} &68.1 (0.7) & \textbf{85.7 (0.6)}\\ 
\text{susy} & 0.5& \textbf{78.2 (0.6)} &76.5 (0.6) &72.6 (1.2) &75.5 (0.4) & \textbf{76.7 (0.9)} &74.0 (0.9)\\ 
\text{cod-rna} &  & \textbf{86.9 (0.5)} &86.0 (0.5) & \textbf{87.3 (0.4)} &80.3 (1.8) & \textbf{87.6 (0.4)} &82.5 (1.4)\\ 
\text{magic} &  &67.4 (1.1) & \textbf{72.0 (1.0)} & \textbf{72.3 (0.8)} &66.2 (1.6) & \textbf{71.5 (0.8)} & \textbf{70.8 (0.7)}\\ 
\hline
\end{tabular}}
\end{table*}

\begin{table*}
\centering
\caption{Mean accuracy and standard error over 10 trials for PU classification, where  $\pi = 0.3$ and $\pitest=0.7$. Outperforming methods are highlighted in boldface using one-sided t-test with the significance level 5\%.}
\scalebox{0.8}{\begin{tabular} { |L|L|L|L|L|L|L|L|L|L|L|L|L|L|}
\hline
\text{Dataset} & \pig& \frac{u}{p}\text{uLSIF}& \frac{p}{u}\text{uLSIF}& \text{DH-Lin}& \text{DH-Ker}& \text{Sq-Lin}& \text{Sq-Ker}\\ \hline
\text{banana} &  & \textbf{87.7 (0.5)} & \textbf{87.9 (0.3)} &67.4 (0.9) &81.8 (0.7) &70.0 (0.0) & \textbf{87.1 (0.4)}\\ 
\text{ijcnn1} &  & \textbf{71.2 (0.3)} & \textbf{71.7 (0.3)} &70.0 (0.0) &70.0 (0.0) &70.9 (0.3) &70.0 (0.0)\\ 
\text{MNIST} & \pitest &77.5 (0.7) &82.5 (0.6) &82.8 (0.5) &82.2 (0.4) &69.0 (1.0) & \textbf{87.1 (0.4)}\\ 
\text{susy} & 0.7 & \textbf{75.6 (0.4)} & \textbf{75.9 (0.5)} &70.0 (0.0) &70.0 (0.0) &73.0 (0.2) &71.9 (0.4)\\ 
\text{cod-rna} &  & \textbf{85.8 (0.7)} &84.7 (0.4) & \textbf{86.2 (0.5)} &78.6 (0.4) & \textbf{86.0 (0.5)} &80.5 (0.5)\\ 
\text{magic} &  &77.0 (0.7) & \textbf{79.0 (0.5)} & \textbf{78.5 (0.6)} &71.1 (0.2) &78.0 (0.4) &76.7 (0.3)\\ 
\hline
\text{banana} &  &83.5 (0.5) & \textbf{88.4 (0.4)} &54.3 (1.5) &83.7 (0.6) &53.8 (1.5) &85.3 (0.6)\\ 
\text{ijcnn1} &  &58.8 (0.8) & \textbf{73.8 (0.6)} & \textbf{72.7 (1.3)} &46.3 (3.0) &71.1 (0.8) &58.6 (0.9)\\ 
\text{MNIST} &  &68.9 (0.9) & \textbf{85.5 (0.4)} &82.2 (0.6) &84.4 (0.5) &68.7 (0.9) & \textbf{86.0 (0.5)}\\ 
\text{susy} & 0.5 &63.7 (0.8) & \textbf{73.7 (0.9)} &68.7 (0.4) &57.2 (0.8) &70.4 (0.5) &66.6 (1.0)\\ 
\text{cod-rna} &  &82.1 (0.7) &83.2 (0.6) & \textbf{85.7 (0.5)} &80.2 (0.7) & \textbf{85.5 (0.4)} &83.1 (0.6)\\ 
\text{magic} &  &63.9 (1.2) & \textbf{78.2 (0.6)} & \textbf{76.5 (0.8)} &69.3 (2.0) & \textbf{77.3 (0.6)} &74.9 (0.9)\\ 
\hline
\text{banana} &  &78.2 (0.7) & \textbf{82.3 (0.5)} &33.4 (1.0) &52.4 (1.8) &30.0 (0.0) &73.6 (1.1)\\ 
\text{ijcnn1} &  & \textbf{38.0 (0.6)} & \textbf{37.8 (0.7)} &30.0 (0.0) &30.0 (0.0) &35.0 (0.4) &30.0 (0.0)\\ 
\text{MNIST} & \pi &60.3 (1.0) &69.8 (0.7) & \textbf{81.2 (0.6)} &62.8 (1.2) &69.7 (0.8) &74.4 (1.0)\\ 
\text{susy} & 0.3 &54.7 (0.8) & \textbf{57.5 (0.9)} &38.6 (0.7) &30.0 (0.0) & \textbf{55.9 (0.4)} &43.8 (0.9)\\ 
\text{cod-rna} & &74.7 (1.0) &78.5 (0.6) & \textbf{82.4 (0.7)} &66.0 (1.3) &79.2 (0.9) &75.2 (1.2)\\ 
\text{magic} &  &45.9 (1.6) &60.6 (1.4) & \textbf{66.3 (1.6)} &30.0 (0.0) &56.9 (1.6) &49.0 (1.5)\\ 
\hline
\end{tabular}}
\end{table*}

\begin{table*}
\centering
\caption{Mean accuracy and standard error over 10 trials for PU classification, where  $\pi = 0.4$ and $\pitest=0.8$. Outperforming methods are highlighted in boldface using one-sided t-test with the significance level 5\%.}
\scalebox{0.8}{\begin{tabular} { |L|L|L|L|L|L|L|L|L|L|L|L|L|L|}
\hline
\text{Dataset} & \pig& \frac{u}{p}\text{uLSIF}& \frac{p}{u}\text{uLSIF}& \text{DH-Lin}& \text{DH-Ker}& \text{Sq-Lin}& \text{Sq-Ker}\\ \hline
\text{banana} &  &89.4 (0.5) & \textbf{90.6 (0.3)} &76.3 (1.5) &84.1 (0.5) &80.0 (0.0) &88.5 (0.3)\\ 
\text{ijcnn1} &  &80.0 (0.0) & \textbf{80.6 (0.1)} &80.0 (0.0) &80.0 (0.0) &80.0 (0.0) &80.0 (0.0)\\ 
\text{MNIST} & \pitest &85.6 (0.6) &86.8 (0.2) &83.3 (0.4) &85.1 (0.4) &67.9 (0.9) & \textbf{88.8 (0.3)}\\ 
\text{susy} & 0.8 &80.0 (0.0) & \textbf{81.3 (0.3)} &80.0 (0.0) &80.0 (0.0) &80.0 (0.0) &80.0 (0.0)\\ 
\text{cod-rna} &  &85.7 (0.5) &86.6 (0.5) & \textbf{88.5 (0.5)} &83.4 (0.5) & \textbf{88.0 (0.5)} &83.2 (0.5)\\ 
\text{magic} & &83.4 (0.5) & \textbf{84.8 (0.4)} & \textbf{84.0 (0.6)} &80.2 (0.1) & \textbf{84.3 (0.9)} &82.2 (0.3)\\ 
\hline
\text{banana} &  &84.2 (0.5) & \textbf{89.6 (0.6)} &74.5 (1.6) &87.4 (0.7) &78.5 (0.9) & \textbf{88.5 (0.4)}\\ 
\text{ijcnn1} &  &77.5 (1.3) & \textbf{80.4 (0.6)} &80.0 (0.0) &79.7 (0.3) & \textbf{81.5 (0.5)} &77.0 (0.6)\\ 
\text{MNIST} &  &72.8 (1.1) & \textbf{88.2 (0.5)} &82.0 (0.6) & \textbf{88.3 (0.6)} &67.5 (1.0) & \textbf{89.2 (0.5)}\\ 
\text{susy} & 0.6 &68.9 (0.7) &77.2 (0.7) & \textbf{80.2 (0.1)} & \textbf{80.0 (0.1)} &76.9 (1.0) &75.6 (0.6)\\ 
\text{cod-rna} &  &82.1 (1.4) &83.9 (0.6) & \textbf{87.0 (0.4)} &82.0 (0.8) & \textbf{87.6 (0.4)} &84.4 (0.7)\\ 
\text{magic} & &76.3 (1.0) & \textbf{82.7 (0.8)} & \textbf{81.5 (0.9)} & \textbf{82.7 (0.4)} & \textbf{82.5 (1.0)} & \textbf{80.7 (1.0)}\\ 
\hline
\text{banana} &  &78.6 (0.7) & \textbf{83.6 (0.8)} &26.1 (1.3) &66.6 (2.4) &21.8 (0.9) &80.9 (0.8)\\ 
\text{ijcnn1} &  &36.2 (0.8) & \textbf{46.5 (1.1)} &20.6 (0.2) &20.0 (0.0) &40.3 (0.6) &22.9 (1.0)\\ 
\text{MNIST} & \pi &57.2 (1.7) &76.7 (1.0) & \textbf{81.3 (0.4)} &75.1 (0.6) &68.7 (1.1) & \textbf{81.2 (0.4)}\\ 
\text{susy} & 0.4 &53.4 (0.3) & \textbf{61.6 (0.7)} &46.2 (1.0) &24.0 (1.2) &56.7 (1.1) &54.4 (0.8)\\ 
\text{cod-rna} &  &75.6 (2.0) &79.4 (0.8) & \textbf{83.0 (0.8)} &74.4 (0.9) & \textbf{83.2 (0.8)} &78.5 (0.8)\\ 
\text{magic} &  &51.9 (1.3) & \textbf{73.4 (0.7)} & \textbf{73.9 (0.7)} &45.4 (3.2) & \textbf{73.9 (0.9)} &66.6 (0.8)\\ 
\hline
\end{tabular}}
\end{table*}

\begin{table*}
\centering
\caption{Mean accuracy and standard error over 10 trials for PU classification, where  $\pi = 0.8$ and $\pitest=0.4$. Outperforming methods are highlighted in boldface using one-sided t-test with the significance level 5\%.}
\scalebox{0.8}{\begin{tabular} { |L|L|L|L|L|L|L|L|L|L|L|L|L|L|}
\hline
\text{Dataset} & \pig& \frac{u}{p}\text{uLSIF}& \frac{p}{u}\text{uLSIF}& \text{DH-Lin}& \text{DH-Ker}& \text{Sq-Lin}& \text{Sq-Ker}\\ \hline
\text{banana} &  &77.3 (1.1) & \textbf{83.2 (0.6)} &59.7 (0.7) &72.3 (2.2) &57.9 (1.6) &77.6 (2.1)\\ 
\text{ijcnn1} &  &57.9 (0.8) &57.9 (0.9) & \textbf{60.1 (0.1)} &59.6 (0.3) & \textbf{61.3 (0.9)} &55.9 (1.3)\\ 
\text{MNIST} & \pitest &74.1 (0.4) &76.9 (0.4) &69.8 (0.7) & \textbf{79.2 (0.8)} &55.2 (1.3) &75.6 (1.0)\\ 
\text{susy} & 0.4&67.2 (0.4) & \textbf{70.5 (0.6)} &66.5 (0.5) &61.2 (0.5) & \textbf{69.0 (1.7)} & \textbf{68.9 (1.0)}\\ 
\text{cod-rna} &  &79.2 (1.1) &82.2 (0.9) & \textbf{87.7 (0.6)} &74.6 (0.8) & \textbf{85.8 (1.4)} &78.9 (1.1)\\ 
\text{magic} &  &63.3 (0.8) &69.8 (1.0) & \textbf{73.1 (0.8)} &65.1 (1.2) & \textbf{74.7 (1.1)} &67.6 (0.8)\\ 
\hline
\text{banana} &  &80.4 (1.4) & \textbf{86.2 (0.5)} &44.1 (2.5) &59.0 (3.1) &43.3 (1.8) &74.4 (1.7)\\ 
\text{ijcnn1} &  &43.8 (0.6) & \textbf{48.0 (2.2)} &40.4 (0.2) &40.7 (0.3) & \textbf{49.7 (0.7)} &43.1 (0.8)\\ 
\text{MNIST} &  &77.8 (0.8) & \textbf{79.9 (0.6)} &69.4 (0.7) &76.5 (1.2) &56.6 (1.3) &74.8 (1.2)\\ 
\text{susy} & 0.6 & \textbf{59.3 (2.6)} & \textbf{62.7 (2.1)} &44.1 (1.1) &41.7 (1.2) & \textbf{61.1 (1.6)} & \textbf{57.5 (2.4)}\\ 
\text{cod-rna} &  &76.6 (1.8) &81.3 (1.2) & \textbf{86.0 (0.7)} &73.3 (0.4) & \textbf{85.1 (0.5)} &73.4 (0.4)\\ 
\text{magic} &  &62.6 (1.2) &67.9 (1.5) & \textbf{71.2 (0.8)} &51.1 (1.7) & \textbf{69.9 (0.8)} &62.8 (1.3)\\ 
\hline
\text{banana} &  & \textbf{77.3 (2.2)} & \textbf{80.8 (1.2)} &40.0 (0.6) &45.0 (1.6) &40.0 (0.0) &57.0 (2.1)\\ 
\text{ijcnn1} &  & \textbf{40.0 (0.0)} & \textbf{40.7 (0.4)} & \textbf{40.0 (0.0)} & \textbf{40.0 (0.0)} & \textbf{40.2 (0.1)} & \textbf{40.0 (0.0)}\\ 
\text{MNIST} & \pi &62.6 (2.3) &61.6 (2.1) & \textbf{68.6 (1.3)} &49.6 (2.2) &62.6 (1.2) & \textbf{67.3 (1.0)}\\ 
\text{susy} & 0.8 & \textbf{40.0 (0.0)} & \textbf{42.8 (1.8)} & \textbf{40.0 (0.0)} & \textbf{40.0 (0.0)} & \textbf{40.8 (0.3)} & \textbf{40.1 (0.1)}\\ 
\text{cod-rna} &  &63.5 (1.8) &66.7 (3.0) & \textbf{74.3 (1.1)} &51.8 (3.2) &70.3 (1.3) &59.3 (2.1)\\ 
\text{magic} &  &44.8 (1.1) & \textbf{57.5 (2.2)} & \textbf{55.2 (1.1)} &40.1 (0.1) & \textbf{57.4 (0.8)} &46.8 (1.1)\\
\hline
\end{tabular}}
\end{table*}
\end{document}